%% file: template.tex
\documentclass{article}
\usepackage{arxiv}

\usepackage[utf8]{inputenc} 
\usepackage[T1]{fontenc}    
\usepackage{hyperref}       
\usepackage{url}            
\usepackage{booktabs}       
\usepackage{amsfonts}       
\usepackage{nicefrac}       
\usepackage{microtype}      
\usepackage{lipsum}		
\usepackage{graphicx}
\usepackage{doi}

\input{dfn}

\title{Learning node embeddings via summary graphs: a brief theoretical analysis}

\author{Houquan Zhou \\
	Data Intelligence Research Center \\
	Institute of Computing Technology \\
	Beijing, 100000 \\
	\texttt{zhouhouquan18z@ict.ac.cn} \\
	\And
	Shenghua liu \\
	Data Intelligence Research Center \\
	Institute of Computing Technology \\
	Beijing, 100000\\
	\texttt{liushenghua@ict.ac.cn} \\
	\And
	Danai Koutra \\
	Computer Science \& Engineering \\
	University of Michigan  \\
	Ann Arbor, MI 48109 USA \\
	\texttt{dkoutra@umich.edu} \\
	\And
	Huawei Shen \\
	Data Intelligence Research Center \\
	Institute of Computing Technology \\
	Beijing, 100000 \\
	\texttt{shenhuawei@ict.ac.cn} \\
	\And
	Xueqi Cheng \\
	Data Intelligence Research Center \\
	Institute of Computing Technology \\
	Beijing, 100000 \\
	\texttt{cxq@ict.ac.cn} \\
}

\renewcommand{\shorttitle}{Learning node embeddings via summary graphs: a brief theoretical analysis}

\hypersetup{
pdftitle={Learning node embeddings via summary graphs: a brief theoretical analysis},
pdfsubject={},
pdfauthor={Houquan Zhou},
}

\begin{document}
\maketitle

\begin{abstract}
	Graph representation learning plays an important role in many graph mining applications, but learning embeddings of large-scale graphs remains a problem.
	Recent works try to improve scalability via graph summarization---i.e., they learn embeddings on a smaller summary graph, and then restore the node embeddings of the original graph. However, all existing works depend on heuristic designs and lack theoretical analysis.

	Different from existing works, we contribute an in-depth theoretical analysis of three specific embedding learning methods based on introduced kernel matrix, and reveal that learning embeddings via graph summarization is actually learning embeddings on a approximate graph constructed by the configuration model.
	We also give analysis about approximation error.
	To the best of our knowledge, this is the first work to give theoretical analysis of this approach.
	Furthermore, our analysis framework gives interpretation of some existing methods and provides great insights for future work on this problem.

\end{abstract}


\input{introduction}
\input{related_work}
\input{preliminary}

\input{analysis}
\input{conclusion}

\bibliographystyle{unsrt}
\bibliography{references}

\end{document}

%% file: dfn.tex
\usepackage{xspace}
\usepackage{adjustbox}
\usepackage[normalem]{ulem}
\usepackage{tablefootnote}
\usepackage{enumitem}
\usepackage{multirow}
\usepackage{threeparttable}
\usepackage{subcaption}
\usepackage{pifont}
\usepackage{amsthm, amssymb}
\usepackage{mathtools}
\usepackage{colortbl}
\usepackage{url}
\usepackage{booktabs}

\newtheorem{theorem}{Theorem}
\newtheorem{lemma}{Lemma}
\newtheorem{corollary}{Corollary}
\newtheorem{definition}{Definition}

\usepackage{algorithm, algorithmic}

\newcommand{\half}{\frac{1}{2}}

\newcommand{\T}{\mathrm{T}}

\definecolor{OliveGreen}{rgb}{0,0.6,0}
\newcommand{\reminder}[1]{{\textsf{\textcolor{blue}{[#1]}}}}

\newcommand{\hide}[1]{}
\newcommand{\mkclean}{\renewcommand{\reminder}{\hide}}
\mkclean

\DeclareMathAlphabet\mathbfcal{OMS}{cmsy}{b}{n}
\newcommand{\MG}{\mathcal{G}}
\newcommand{\MV}{\mathcal{V}}
\newcommand{\ME}{\mathcal{E}}
\newcommand{\SN}{\mathcal{S}}
\newcommand{\KF}{\mathcal{K}}

\newcommand{\BD}{\mathbf{D}}
\newcommand{\BA}{\mathbf{A}}

\newcommand{\BI}{\mathbf{I}}
\newcommand{\BP}{\mathbf{P}}
\newcommand{\BQ}{\mathbf{Q}}
\newcommand{\BR}{\mathbf{R}}
\newcommand{\BE}{\mathbf{E}}
\newcommand{\BX}{\mathbf{X}}
\newcommand{\BY}{\mathbf{Y}}
\newcommand{\BW}{\mathbf{W}}
\newcommand{\BM}{\mathbf{M}}

\newcommand{\NA}{\mathbf{\mathcal{A}}}

\newcommand{\restorefunc}{f}

%% file: introduction.tex
\section{Introduction}
Graph representation learning has gained much research interest in recent years due to its success in various fields including biology, computer vision, text classification, and more.
However, many representation learning methods are hard to scale to large graphs.

To overcome this problem, some researchers employ matrix factorization technique~\cite{ou2016asymmetric,qiu2018network,qiu2019netsmf} or sampling methods~\cite{hamilton2017inductive,chen2018fastgcn,huang2018adaptive,ying2018webscale,ladies2019,chiang2019clustergcn,zeng2020graphsaint}.
Other works first summarize the input graph into a smaller summary graph by grouping subsets of nodes into supernodes and linking them via superedges~\cite{liu2018graph}; these approaches, thus, reduce the number of nodes.
Then, they employ embedding learning methods on the summary graph and restore the node embeddings of the original graph.
Representative works include HARP~\cite{chen2018harp}, MILE~\cite{liang_mile_2020}, and GraphZoom~\cite{deng2020graphzoom}.

A key limitation of the existing summarization-based solutions for scaling up graph representation learning is that they leverage heuristic summarization methods, 
and then \textbf{empirically} restore the embeddings of the original nodes from the supernode embeddings.
Thus, there are no theoretical studies of the underlying mechanisms.


In this work, we reveal the theoretical mechanism of learning embedding via summary graphs.
We theoretically show that applying three methods (DeepWalk, LINE, and GCN) on summary graphs is equivalent to applying them on a reconstructed graph based on the configuration model.
The main theoretical results is summarized in Table~\ref{tab:summary}.
To the best of our knowledge, this paper is the first to analyze the problem theoretically.

The rest of the paper is organized as follows: In Section~\ref{sec:related_work}, we introduce some existing works of embedding learning by summarization. Section~\ref{sec:preliminary} presents some basic concepts about graph summarization and node embeding learning methods. The main theoretical analysis is give in Section~\ref{sec:analysis}, including the definition of kernel matrix, the connection of summary graphs and original graphs, and the approximation error. Section~\ref{sec:conclusion} concludes the paper.

\definecolor{Gray}{gray}{0.85}
\begin{table}[!t]
\centering
\renewcommand{\arraystretch}{1.15}
\begin{threeparttable}
    \caption{
    The closed-form solutions for learning original graph embeddings from smaller, summary graphs, in DeepWalk, LINE, and GCN, using restoration matrix $\BR$. $\BE$ is embeddings on original graph $\MG$, and $\BE_s$ is embeddings on summary graph $\MG_s$.
    }\label{tab:summary}
    \begin{tabular}{l>{\centering}p{0.28\columnwidth}>{\centering}p{0.2\columnwidth}cc}
        \toprule
        \multirow{2}{*}{\textbf{Method}} & \textbf{Kernel matrix} & \textbf{Restoration}  & \textbf{Embeddings} \\
         & $\KF(\MG)$ & {$\BR(i,p), v_i \in \SN_p$\tnote{1}} & $\BE$ \\
        \midrule
        DeepWalk & {$(\BD^{-1} \BA)^{\tau} \BD^{-1} $} & 1 & \multirow{4}{*}{$\BR \BE_s$} \\
        LINE & {$\BD^{-1} \BA \BD^{-1} $} & 1 &   \\
        GCN & {$\tilde{\BD}^{-\frac{1}{2}} \tilde{\BA} \tilde{\BD}^{-\frac{1}{2}}$} & $\sqrt{\frac{d_i}{d_p^{(s)}}}$ &  \\
        \rowcolor{Gray} General form & {$\left(\BD^{-c} \BA \BD^{-1+c}\right)^{\tau}$ $\BD^{1-2c}$} & $\left(\frac{d_i}{d_{p}^{(s)}}\right)^{1-c}$ & \\
        \bottomrule
    \end{tabular}
    \begin{tablenotes}
        \item[1]{\footnotesize $v_i$: Node $i$ in original graphs, $\SN_p$: Supernode $p$ in summary graphs.}
    \end{tablenotes}
\end{threeparttable}
\end{table}

%% file: related_work.tex
\section{Related Work}\label{sec:related_work}
\textbf{Graph representation learning}. Graph representation learning aims to map each node in graphs into a low-dimensional vector (called embedding or representation) which captures the structural information.
The learned latent embeddings can then be fed into machine learning and data mining algorithms for various downstream tasks, such as node classification and link prediction.
A few representative examples from the rich literature in graph representation learning include:
DeepWalk~\cite{perozzi2014deepwalk}, node2vec~\cite{grover2016node2vec}, LINE~\cite{tang2015line}, and graph neural network (GNN) methods~\cite{zhou2019graph,wu2021comprehensive}, such as GCN~\cite{kipf2017semi}, GraphSAGE~\cite{hamilton2017inductive} and GAT~\cite{velickovic2018graph}, which adopt a message-passing framework and update the node embeddings based on their neighbors' representations recursively.

Although graph representation learning methods are successful, their lack of scalability and efficiency is an important problem. To tackle this problem, some works employ sampling techniques, including \emph{layer sampling}~\cite{chen2018stochastic,chen2018fastgcn,huang2018adaptive} and \emph{subgraph sampling}~\cite{ying2018webscale,chiang2019clustergcn,zeng2020graphsaint}.



\vspace{0.2cm}
\noindent \textbf{Embedding learning by summarization}. Another way to improve the scalability is via \emph{graph summarization}~\cite{yan2019groupinn,liu2018graph}. The typical approach is to coarsen the original graph into a smaller summary graph, and apply representation learning methods on it to obtain intermediate embeddings. The embeddings of the original nodes are then restored with a further refinement step.
For example, HARP~\cite{chen2018harp} finds a series of smaller graphs which preserve the global structure of the input graph, and learns representations hierarchically.
HSRL~\cite{fu2019learning} learns embeddings on multi-level summary graphs, and concatenate them to restore original embeddings.
MILE~\cite{liang_mile_2020} repeatedly coarsens the input graph into smaller ones using a hybrid matching strategy, and finally refines the embeddings via GCN to obtain the original node embeddings.
GPA~\cite{lin2020initialization} uses METIS~\cite{karypis1998metis} to partition the graphs, and smooths the restored embeddings via a propagation process.
GraphZoom~\cite{deng2020graphzoom} employs an extra graph fusion step to combine the structural information and feature information, and then uses a spectral coarsening method to merge nodes based on their spectral similarities. Embeddings are then refined by a graph filter to ensure feature smoothness.
\cite{fahrbach2020fast} learns embeddings of the given subset of nodes by coarsening the remaining nodes, which is not capable to learn embeddings of the remaining ones.

However, all these methods can only apply to unsupervised methods.
Moreover, they are based on heuristic designs and lack theoretical formulation and analysis.
This paper aims to fill this gap and give some theoretical analysis of this problem.



\renewcommand{\arraystretch}{1}

%% file: preliminary.tex
\begin{table}[thb]
   \small
    \centering
    \caption{Major Symbols and Definitions.}\label{tab:symbols}
    \begin{tabular}{ll}
        \toprule
        \textbf{Symbol} & \textbf{Definition} \\
        \midrule
        $\MG$=$(\MV, \ME)$ & Original graph with nodeset $\MV$ and edgeset $\ME$ \\
        $\MG_s$=$(\MV_s, \ME_s)$ & Summary graph with supernodes $\MV_s$ and superedges $\ME_s$ \\
        $\MG_r$=$(\MV, \ME_r)$ & Reconstructed graph with nodeset $\MV$ and edgeset $\ME_r$ \\
        $v_i$ & Node $i$ in the original graph $\MG$ \\
        $\SN_p$ & Supernode $p$ in the summary graph $\MG_s$ \\
        $d_i, d_p^{(s)}$ & Degree of node $i$ and supernode $p$ \\ \midrule
        $\BA, \BA_s, \BA_r$ & Adjacency matrix of original, summary, reconstructed graph \\
        $\BD, \BD_s$ & Degree matrix of original and summary graph \\
        $\BP, \BQ$ & Membership and reconstruction matrix in summarization \\
        $\BR$ & Restoration matrix for recovering the original embeddings \\
        $\BE, \BE_{s}$ & Embeddings of original graph and summary graph \\
        \bottomrule
    \end{tabular}
\end{table}
\section{Preliminary}\label{sec:preliminary}
In this section, we first introduce some basic concepts of graph embedding learning and graph summarization.
Table~\ref{tab:symbols} gives the most frequently used symbols in the paper.

\subsection{Graph Embedding}
\label{sec:pre-embedding}
In this paper, we theoretically analyze three graph embedding methods, DeepWalk, LINE, and GCN. 

\subsubsection{DeepWalk and LINE}
DeepWalk~\cite{perozzi2014deepwalk} is an unsupervised graph representation learning method inspired by the success of word2vec in text embedding.
It generates random walk sequences and treats them as sentences that are later fed into a skip-gram model with negative sampling to learn latent node representations.

It has been proved in~\cite{qiu2018network} that DeepWalk is implicitly approximating and factorizing the following matrix:
\begin{equation}\label{equ:deepwalk}
    \BM \coloneqq \log \left( \mathrm{vol}(\MG) \left( \frac{1}{T} \sum_{\tau=1}^{T} (\BD^{-1} \BA)^{\tau}\right) \BD^{-1}\right) - \log b,
\end{equation}
where $T$ and $b$ are the context window size and the number of negative samples in DeepWalk, respectively.

LINE~\cite{tang2015line} learns embeddings by optimizing a carefully designed objective function that aims to preserve both the first-order and second-order proximity.
Though LINE and DeepWalk appear to be different, it has been shown in~\cite{qiu2018network} that LINE is also equivalent to factorizing a similar matrix to Eq.~\eqref{equ:deepwalk} and is a special case of DeepWalk for $T=1$:
\begin{equation}\label{equ:LINE}
    \BM \coloneqq \log \left( \mathrm{vol}(\MG) \BD^{-1} \BA \BD^{-1} \right) - \log b.
\end{equation}

\subsubsection{GCN}
GCN~\cite{kipf2017semi} is a graph neural network model transferring traditional convolution neural network to non-Euclidean graph data. In each layer of GCN, node features are propagated based on a first-order approximation of spectral convolutions on graphs:
\begin{equation}\label{equ:gcn_conv}
    \BE^{(k+1)} = \sigma(\tilde{\BD}^{-\frac{1}{2}} \tilde{\BA} \tilde{\BD}^{-\frac{1}{2}} \BE^{(k)} \BW^{(k)}) \, ,
\end{equation}
where $\BE^{(k)}$ are the node embeddings at the $k$-th layer,  $\BE^{(0)} = \BX$ is the input node feature matrix,
$\BW^{(k)}$ is a learnable weight matrix at the $k$-th layer,
$\tilde{\BA}=\BA+\BI$ is the augmented adjacency matrix with self-loops ($\BI$ is the identity matrix), and
$\tilde{\BD}$ is the corresponding augmented degree matrix.
Finally, $\sigma(\cdot)$ is the non-linear ReLU operation as an activation function, i.e., $\sigma(x) = \max(0, x)$.

\subsection{Graph Summarization}\label{sec:pre_gs}
Given an input graph $\MG=(\MV, \ME)$ with $n=|\MV|$ nodes, graph summarization aims to find a smaller summary graph $\MG_s=(\MV_s, \ME_s)$ (with $n_s = |\MV_s|$ nodes) that preserves the structural information of the original graph.
The supernode set $\MV_s$ forms a partition of the original node set $\MV$ such that every node $v \in \MV$ belongs to exactly one supernode $\SN\in\MV_s$.
The supernodes are connected via superedges $\ME_s$, which are weighted by the sum of original edges between the constituent nodes.
That is, superedge $\BA_s(p, q)$ between supernodes $\SN_p$, $\SN_q$ is defined as:
\begin{equation*}
    \BA_s(p, q) = \sum_{v_i\in \SN_p} \sum_{v_j\in \SN_q} \BA(i, j).
\end{equation*}

The adjacency matrix of the summary graph can be formulated using a \emph{membership matrix} $\BP\in \mathbb{R}^{n_s\times n}$ as $\BA_s = \BP \BA \BP^{\T}$, where
\begin{equation*}
    \BP(p, i) = \begin{cases}
        1 & \quad \text{if } v_i \in \SN_p \\
        0 & \quad \text{otherwise}.
    \end{cases}
\end{equation*}
Given the summmary graph $\MG_s$, the original graph $\MG$ can be approximated with the reconstructed graphs $\MG_r$ with adjacency matrix $A_r$ defined as:
\begin{equation}
\label{eq:pre-reconstruction}
    \BA_r = \BQ \BA_s \BQ^\T \, ,
\end{equation}
where $\BQ\in \mathbb{R}^{n\times n_s}$ is the \emph{reconstruction matrix}. Note that $\BA_r$ can be seen as a \textbf{low-rank approximation} of the original $\BA$.

Specifically, in this work, we consider the configuration-based reconstruction scheme~\cite{zhou2021dpgs}, which adopts \emph{the configuration-based model} as null model. In that case, reconstructed edge weights are proportional to degrees of endpoints. $\BQ$ and $\BA_r$ are defined as:
\begin{align}
   \BQ(i, p) &= \begin{cases}
               \frac{d_i}{d_p^{(s)}} & \quad \text{if } v_i \in \SN_p \\
               0               & \quad \text{otherwise}
           \end{cases} \label{equ:cr_q} \\
   \BA_r(i, j) &= \frac{d_i}{d_p^{(s)}} \BA_s(p, q) \frac{d_j}{d_q^{(s)}}\quad v_i\in \SN_p, v_j\in \SN_q \label{equ:cr_ar}
\end{align}
We will later show that, this reconstruction scheme plays an important role in our theoretical analysis.

%% file: analysis.tex
\section{Theoretical analysis}\label{sec:analysis}
In this section, we theoretically reveal the mechanism behind the approach of learning embeddings on summary graphs.
In short, we show that running three embedding methods (DeepWalk, LINE and GCN) on a summary graph is equivalent to running them on a \textbf{approximate configuration-based reconstructed graph}.

\subsection{Approximating kernel matrices}\label{sec:method_kernel}
We begin our analysis with the following \textbf{kernel matrix}.
By comparing the function forms of DeepWalk, LINE and GCN, we observe that a common kernel matrix can be summarized as:
\begin{definition}[Kernel Matrix]\label{def:kernel_matrix} DeepWalk, LINE, and GCN are based on the following generalized kernel matrix:
\begin{equation}\label{equ:kernel_matrix}
 \KF_\tau(\MG) \coloneqq \left( \BD^{-c} \BA \BD^{-1+c}\right)^{\tau} \BD^{1-2c},
\end{equation}
where $0\le c \le 1$ and $\tau$ is a positive integer, and $\BA$ and $\BD$ are adjacency matrix and degree matrix of $\MG$ respectively. We omit the subscript $\tau$ if there is no ambiguity. For $c=1$, we obtain the matrix that appears in DeepWalk and LINE, and $c=\frac{1}{2}$ yields the matrix within the GCN formulation.
\end{definition}

As we show next in Lemma~\ref{thm:approx}, under \textbf{the configuration-based reconstruction scheme} (see Eq.~\eqref{equ:cr_q} and~\eqref{equ:cr_ar}), this kernel matrix on the original graph, $\KF(\MG)$, can be approximated with the same kernel matrix on the summary graph, $\KF(\MG_s)$, in a closed form.
\begin{theorem}\label{thm:approx}
    Given $\BA_r$ (reconstructed by the configuration-based scheme, see Eq.~\eqref{equ:cr_ar}) as a low-rank approximation of the original adjacency matrix $\BA$, the kernel matrix of $\MG$ can be approximated by the one on $\MG_s$ as follow:
    \begin{equation}\label{eq:kernel-approx}
        \begin{aligned}
        \KF(\MG)  &\approx \left( \BD^{-c} \BA_r \BD^{-1+c}\right)^{\tau} \BD^{1-2c} \\
            &= \BR \left(\BD_s^{-c} \BA_s \BD_s^{-1+c} \right)^{\tau} \BD_s^{1-2c} \BR^\T \\
            &= \BR~~\KF(\MG_s)~~\BR^\T,
        \end{aligned}
    \end{equation}
    where $\BR \in \mathbb{R}^{n\times n_s}$ is the restoration matrix:
    \begin{equation*}
        \BR(i, p) = \begin{cases}
            \left(\frac{d_i}{d_{p}^{(s)}}\right)^{1-c} &\quad \text{if } v_i \in \SN_p \\
            0 &\quad \text{otherwise},
        \end{cases}
    \end{equation*}
    which is closely related to the configuration-based reconstruction matrix $\BQ$ given in \eqref{equ:cr_q}.
\end{theorem}
\begin{proof}
Before we prove Lemma 1, we first introduce Lemma~\ref{lem:A1} and Lemma~\ref{lem:A2}.
\begin{lemma}\label{lem:A1}
    \begin{equation}
        \BQ^{\T} \BD^{-1} \BQ = \BD_s^{-1} \, ,
    \end{equation}
    where $\BQ$ is the reconstruction matrix in the configuration-based reconstruction scheme~(Eq.~\eqref{equ:cr_q}), $\BD$ and $\BD_s$ are degree matrix of the original graph and the summary graph.
\end{lemma}
\begin{proof}
    The $(p, q)$-th entry in $\BQ^{\T} \BD^{-1} \BQ$ is:
    \begin{equation*}
        \BQ^{\T} \BD^{-1} \BQ(p, q) = \sum_{i} \BQ(i, p) \frac{1}{d_i} \BQ(i, q)
    \end{equation*}
    It is easy to see that the result is not zero only when $p = q$ (since a node $v_i$ cannot belongs to two supernodes $\SN_p$ and $\SN_q$ simultaneously). And diagonal items are (note that $d_p^{(s)} = \sum_{v_i\in \SN_p} d_i$):
    \begin{equation*}
        \begin{aligned}
            \BQ^{\T} \BD^{-1} \BQ(p, p) &= \sum_{v_i\in \SN_p} \BQ(i, p) \frac{1}{d_i} \BQ(i, p) = \sum_{v_i\in \SN_p} \frac{d_i}{d_p^{(s)}} \frac{1}{d_i} \frac{d_i}{d_p^{(s)}} \\
            &= \sum_{v_i\in \SN_p} \frac{d_i}{d_p^{(s)}} \frac{1}{d_p^{(s)}} = \frac{1}{d_p^{(s)}} = \BD_s^{-1}(p,p)
        \end{aligned}
    \end{equation*}
\end{proof}
\begin{lemma}\label{lem:A2}
    \begin{equation}
        \BR \BD_s^{-c} = \BD^{-c} \BQ
    \end{equation}
\end{lemma}
\begin{proof}
    Suppose $v_i \in \SN_p$, then the $(i, p)$-th entry of $\BR \BD_s^{-c}$ is:
    \begin{equation*}
        \BR \BD_s^{-c} (i, p) = \left( \frac{d_i}{d_p^{(s)}} \right)^{1-c} \left(d_p^{(s)}\right)^{-c}  = \frac{d_i^{1-c}}{d_p^{(s)}}
    \end{equation*}
    And the $(i, p)$-th entry of $\BD^{-c} \BQ$ is:
    \begin{equation*}
            \BD^{-c} \BQ (i, p) = (d_i)^{-c} \frac{d_i}{d_p^{(s)}} = \frac{d_i^{1-c}}{d_p^{(s)}}
    \end{equation*}
    Thus $\BR \BD_s^{-c} = \BD^{-c} \BQ$.
\end{proof}

Now we prove Theorem \ref{thm:approx}.
Denote $\KF_{\tau}(\MG) = \left( \BD^{-c} \BA_r \BD^{-1+c}\right)^{\tau} \BD^{1-2c}$ for convenience.

Prove by induction. When $\tau = 1$,
\begin{equation*}
    \begin{aligned}
        \KF_{1}(\MG_r) &= \BD^{-c} \BA_r \BD^{-1+c} \BD^{1-2c} \\
        &= \BD^{-c} \BQ \BA_s \BQ^{\T} \BD^{-c}  \\
        &= \BR \BD_s^{-c} \BA_s \BD_s^{-c} \BR^{\T} \qquad \text{(Lemma \ref{lem:A2})}\\
        &= \BR~~\KF_{1}(\MG_s)~~\BR^{\T}
    \end{aligned}
\end{equation*}
Suppose the the lemma holds for $\tau=i$, i.e., $\KF_{i}(\MG_r) = \BR~~\KF_{i}(\MG_s)~~\BR^{\T}$. For the case $\tau = i+1$,
\begin{align*}
        \KF_{i+1}(\MG_r) &=  \BD^{-c} \BA_r \BD^{-1+c}~~\KF_{i}(\MG_r) \\
        &= \BD^{-c} \BA_r \BD^{-1+c} \BR~~\KF_{i}(\MG_s)~~\BR^{\T} \\
        &= \BD^{-c} \BQ \BA_s \BQ^{\T} \BD^{-1+c} \BR~~\KF_{i}(\MG_s)~~\BR^{\T} \\
        &= \BD^{-c} \BQ \BA_s \BQ^{\T} \BD^{-1} (\BD^{c} \BR)~~\KF_{i}(\MG_s)~~\BR^{\T} \\
        &\footnotemark= \BD^{-c} \BQ \BA_s \BQ^{\T} \BD^{-1} (\BQ \BD_s^{c})~~\KF_{i}(\MG_s)~~\BR^{\T} \\
        \shortintertext{(Lemma \ref{lem:A1} and \ref{lem:A2}.)}
        &= \BR \BD_s^{-c} \BA_s \BD_s^{-1+c}~~\KF_{i}(\MG_s)~~\BR^{\T} \\
        &= \BR~~\KF_{i+1}(\MG_s)~~\BR^{\T}
\end{align*}
Applying principal of induction finishes the proof.
\end{proof}
\footnotetext{$\BD^{-c} \BQ = \BR \BD_s^{-c} \Rightarrow \BQ = \BD^{c} \BR \BD_s^{-c} \Rightarrow \BQ \BD_s^{c} = \BD^{c} \BR$.}

From this general form of the reconstruction matrix, we obtain specific cases for DeepWalk, LINE, and GCN in the next corollaries.
\begin{corollary}\label{coro:rw_matrix}
Based on Dfn.~\ref{def:kernel_matrix} of the kernel matrix, $c=1$ corresponds to DeepWalk and LINE. In this case, \eqref{eq:kernel-approx} becomes:
{
    \begin{equation}
        \begin{aligned}
            \KF(\MG) &= \left( \BD^{-1} \BA \right)^{\tau} \BD^{-1} \approx \left( \BD^{-1} \BA_r \right)^{\tau} \BD^{-1} \\
            &= \BR \left(\BD_s^{-1} \BA_s \right)^{\tau} \BD_s^{-1} \BR^\T \, ,
        \end{aligned}
    \end{equation}
}
    where
    \begin{equation}\label{equ:rw_R}
        \BR(i, p) = \begin{cases}
            1 &\quad \text{if } v_i \in \SN_p \\
            0 &\quad \text{otherwise}.
        \end{cases}
    \end{equation}
\end{corollary}

\begin{corollary}\label{coro:symm_laplacian}
Based on Dfn.~\ref{def:kernel_matrix} of the kernel matrix, $c=\frac{1}{2}$ corresponds to GCN. In this case, \eqref{eq:kernel-approx} becomes:
{
    \begin{equation}
        \begin{aligned}
            \KF(\MG) &= \BD^{-\frac{1}{2}} \BA \BD^{-\frac{1}{2}} \approx  \BD^{-\frac{1}{2}} \BA_r \BD^{-\frac{1}{2}} \\
            &= \BR~ \left(\BD_s^{-\frac{1}{2}} \BA_s \BD_s^{-\frac{1}{2}}\right)~\BR^\T  \, ,
        \end{aligned}
    \end{equation}
}
    where
    \begin{equation}\label{equ:symm_R}
        \BR(i, p) = \begin{cases}
            \sqrt{ \frac{d_i}{d_{p}^{(s)}} } &\quad \text{if } v_i \in \SN_p \\
            0 &\quad \text{otherwise}.
        \end{cases}
    \end{equation}
\end{corollary}
Note that $\BR^\T \BR = \mathbf{I}$ and $\BR^{\dagger} = \BR^{\T}$ ($\BR^{\dagger}$ denotes the Moore-Penrose inverse of $\BR$) in Corollary~\ref{coro:symm_laplacian}, which is important in our analysis of GCN.

\subsection{Error analysis} One may ask the question that how much the error of kernel matrix is introduced by replacing $\BA$ by $\BA_r$? Theorem \ref{thm:bound} gives an brief analysis.
\begin{theorem}\label{thm:bound}
    By replacing $\BA$ by $\BA_r$, the error of kernel matrix is bounded by:
    \begin{equation}
        \| \KF_\tau(G) - \KF_\tau(\MG_s) \|_F \le d_{min}^{-1-2c} \cdot \tau \cdot \| \BD^{-\half} \BA \BD^{-\half} - \BD^{-\half} \BA_r \BD^{-\half} \|_F
    \end{equation}
    where $d_{min}$ is the minimum degree.
\end{theorem}
\begin{proof}
    Note that the kernel matrix $\KF_\tau(G)$ can be rewritten as:
    \begin{equation}
        \KF_\tau(G) = \BD^{\half-c} (\BD^{-\half} \BA \BD^{-\half})^\tau \BD^{\half-c}
    \end{equation}
    Then,
    \begin{equation}
        \begin{aligned}
            &\| \KF_\tau(\MG) - \KF_\tau(\MG_r)\|_F \\
            =& \left\| \BD^{\half-c} \left( (\BD^{-\half} \BA \BD^{-\half})^\tau - (\BD^{-\half} \BA_r \BD^{-\half})^\tau \right) \BD^{\half-c} \right\|_F \\
            =& \left\| \BD^{-\half-c}\right\|_2^2 \cdot \left\| (\BD^{-\half} \BA \BD^{-\half})^\tau - (\BD^{-\half} \BA_r \BD^{-\half})^\tau \right\|_F \\
            =& \ d_{min}^{-1-2c} \cdot \left\| (\BD^{-\half} \BA \BD^{-\half})^\tau - (\BD^{-\half} \BA_r \BD^{-\half})^\tau \right\|_F \\
        \end{aligned}
    \end{equation}

    Denote $\NA = \BD^{-\half} \BA \BD^{-\half}$ and $\NA_r = \BD^{-\half} \BA_r \BD^{-\half}$ for notation simplicity, we have
    \begin{equation}
        \NA^\tau - \NA_r^\tau = (\NA^{\tau-1} - \NA_r^{\tau-1})\NA + \NA_r^{\tau-1}(\NA - \NA_r)
    \end{equation}
    And,
    \begin{align*}
        \| \NA^\tau - \NA_r^\tau \|_F &\le \|\NA (\NA^{\tau-1} - \NA_r^{\tau-1})\|_F + \| \NA_r^{\tau-1} (\NA-\NA_r) \|_F \\
        &\le \|\NA\|_2 \| \NA^{\tau-1} - \NA_r^{\tau-1} \|_F + \| \NA_r \|_2^{\tau-1} \| \NA - \NA_r \|_F \\
        \intertext{($\| \NA \|_2 \le 1 $ and $\| \NA_r \|_2 \le 1$)}
        &\le \| \NA^{\tau-1} - \NA_r^{\tau-1} \|_F + \| \NA - \NA_r \|_F
    \end{align*}
    Appling it recursively, we have:
    \begin{equation}
        \| \NA^\tau - \NA_r^\tau \|_F \le \tau \| \NA - \NA_r \|_F
    \end{equation}
    Thus,
    \begin{equation}
        \| \KF_\tau(\MG) - \KF_\tau(\MG_r) \|_F \le d_{min}^{-1-2c} \cdot \tau \cdot \left\| \BD^{-1/2} \BA \BD^{-1/2} - \BD^{-1/2} \BA_r \BD^{-1/2} \right\|_F
    \end{equation}
    where $d_{min} = \min_i d_i$ is the minimum degree.
\end{proof}

\subsection{Approximating DeepWalk / LINE}\label{sec:method_dw}
Based on Corollary \ref{coro:symm_laplacian}, we now discuss how to approximate the DeepWalk and LINE node embeddings for the original nodes.
Since LINE is a special case of DeepWalk, we focus on the former; similar conclusions can be easily drawn for LINE.

\begin{theorem}\label{thm:approx_dw}
    Embeddings learned by DeepWalk on the original graph $\MG$, $\BE$, can be approximated by embeddings learned by DeepWalk on the summary graph $\MG_s$, $\BE_{s}$, using the restoration matrix $\BR$ in \eqref{equ:rw_R}, i.e.,
    \begin{equation}
        \BE \approx \restorefunc(\BE_s) = \BR~\BE_s
    \end{equation}
\end{theorem}
\begin{proof}
    Consider $\BA_r$ as a low-rank approximation of $\BA$, and replace $\BA$ by $\BA_r$ in the DeepWalk matrix. According to Corollary \ref{coro:rw_matrix}:
\begin{equation*}\label{equ:deepwalk_summ}
    \begin{aligned}
        \BM &= \log \left( \frac{\mathrm{vol}(\MG)}{bT} \sum_{\tau=1}^{T} (\BD^{-1} \BA)^{\tau} \BD^{-1} \right) \\
        &\approx \log \left( \frac{\mathrm{vol}(\MG)}{bT} \sum_{\tau=1}^{T} (\BD^{-1} \BA_r)^{\tau} \BD^{-1} \right) \\
        &= \log \left( \frac{\mathrm{vol}(\MG)}{bT} \BR \left( \sum_{\tau=1}^{T} (\BD_s^{-1} \BA_s)^{\tau} \BD_s^{-1} \right) \BR^{\T} \right) \\
        &\footnotemark= \, \BR\cdot \log \left( \frac{\mathrm{vol}(\MG)}{bT} \left( \sum_{\tau=1}^{T} (\BD_s^{-1} \BA_s)^{\tau} \BD_s^{-1} \right) \right) \cdot \BR^{\T} \\
        &= \BR~~\BM_s~~\BR^{\T} \, ,
    \end{aligned}
\end{equation*}
\footnotetext{This equation holds since each row of $\BR$ contains exactly one non-zero value "1". Thus we can take it out of the $\log$ function.}
where $\BM_s$ is the corresponding matrix DeepWalk factorizing on summary graph $\MG_s$.

Suppose $\BM_s$ is factorized into $\BM_s = \BX_s \BY_s^{\T}$, then $\BM \approx (\BR \BX_s) (\BR \BY_s)^{\T}$. That is, embeddings of original graph $\MG$ can be approximated by embeddings learned on summary graph $\MG_s$ with a restoration matrix $\BR$.
\begin{equation}\label{equ:dw_restore}
    \BE \approx \BR\cdot \BE_{s}
\end{equation}
\end{proof}

According to Theorem \ref{thm:approx_dw} and the definition of $\BR$ matrix ($\BR(i, p) = 1$ if $v_i \in \SN_p$), we can conclude that nodes in the same supernode get the same embeddings after the restoration.
This approach, is exactly the way how related works (including HARP, MILE and GraphZoom) restore the embeddings. Thus, Theorem~\ref{thm:approx_dw} provides a \textbf{theoretical interpretation} for the restoration step of existing methods.

\subsection{Approximating GCN}\label{sec:method_gcn}
Given the embeddings $\BE_s$ learned by GCN on the summary graph which are usually the output of the last convolution layer, i.e.~$\BE_s=\BE_s^{(K)}$, we can approximate original node embeddings, $\BE$, as stated in the following theorem.
\begin{theorem}\label{thm:approx_gcn}
    Embeddings learned by GCN on the original graph $\MG$ can be approximated by embeddings learned by GCN on the summary graph $\MG_s$ with initial features $\BX_s \coloneqq \BR^{\T} \BX$, using the restoration matrix $\BR$ defined in \eqref{equ:symm_R}, in a \textbf{least-square approximation} perspective:
    \begin{equation}
        \BE \approx \restorefunc(\BE_s) = \BR~\BE_s
    \end{equation}
    And under reasonable assumption, the reconstruction error of embeddings is bounded by:
    {
        \small
        \begin{equation}\label{equ:gcn_bound}
            \left\| \BE - \BR \BE_s \right\| \le \left\| \tilde{\BD}^{-\frac{1}{2}} (\tilde{\BA} - \tilde{\BA}_r) \tilde{\BD}^{-\frac{1}{2}} \right\| \left( \prod_{l=0}^{K-1} \left\| \BW^{(l)} \right\| \right) \left\| \BE^{(0)} \right\|
        \end{equation}
    }
    The first term is the difference of the normalized adjacency matrix between $\MG$ and $\MG_r$.
\end{theorem}
\begin{proof}
       Consider $\BA_r$ as a low-rank approximation of $\BA$, and replace $\BA$ by $\BA_r$ in the $k$-th layer of GCN. Further assume that the weight matrices of GCN on original graphs and summary graphs are the same. According to Corollary \ref{coro:symm_laplacian}:
    \begin{equation}\label{equ:layer_approx}
        \begin{aligned}
            \BE^{(k+1)} &= \sigma(\tilde{\BD}^{-\frac{1}{2}} \tilde{\BA} \tilde{\BD}^{-\frac{1}{2}} \BE^{(k)} \BW^{(k)}) \\
            &\approx \sigma(\tilde{\BD}^{-\frac{1}{2}} \tilde{\BA}_r \tilde{\BD}^{-\frac{1}{2}} \BE^{(k)} \BW^{(k)}) \\
            &= \sigma(\BR (\tilde{\BD}_s^{-\frac{1}{2}} \tilde{\BA}_s \tilde{\BD}_s^{-\frac{1}{2}}) (\BR^\T \BE^{(k)}) \BW^{(k)}) \\
            &\footnotemark= \, \BR\cdot \sigma( (\tilde{\BD}_s^{-\frac{1}{2}} \tilde{\BA}_s \tilde{\BD}_s^{-\frac{1}{2}}) (\BR^\T \BE^{(k)}) \BW^{(k)}) 
        \end{aligned}
    \end{equation}
    Let $\BE_t^{(k)} = \BR^{\T} \BE^{(k)}$. Note that
    $\BR^{\T} \BR = \mathbf{I}$, we have
    \begin{equation}\label{eq:apprGCNsumm}
        \begin{aligned}
            \BE_t^{(k)} &= \BR^{\T} \BE^{(k)}  \\
            &\approx \sigma((\tilde{\BD}_s^{-\frac{1}{2}} \tilde{\BA}_s \tilde{\BD}_s^{-\frac{1}{2}}) \BE_{t}^{(k-1)} \BW^{(k-1)})
        \end{aligned}
    \end{equation}

    \footnotetext{This equation holds since each row of $\BR$ contains only one non-zero value. That's why we can take $\BR$ out of the $\sigma$ function.}

    Note that approximate equation \eqref{eq:apprGCNsumm} is a GCN convolution layer on the summary graph $\MG_s$.

    By optimizing a GCN network on summary graph
    $\MG_s$ with initial feature $\BX_s \coloneqq \BR^{\T} \BX$, we can get exact embedding solution denoted as $\BE_s^{(k)}$.
    Then we have
    \begin{equation*}
        \BE_s^{(k)} \approx \BE_t^{(k)} = \BR^{\T} \BE^{(k)}.
    \end{equation*}
    And then given $\BE_s^{(k)}$ and $\BR$, we can solve original embedding $\BE^{(k)}$ using a \textbf{least-square approximation}, that is:
    \begin{equation*}
        \BE^{(k)} \approx (\BR^{\T})^{\dagger} \BE_s^{(k)} = \BR~\BE_s^{(k)}
    \end{equation*}
    Therefore, we have $\BE = \BE^{(K)} \approx \BR~\BE_s^{(K)} = \BR~\BE_s$.

    Next we will prove the bound in Eq~\eqref{equ:gcn_bound}. In the following proof, $\| \cdot \|$ denotes the Frobenius norm of matrix.

    Suppose that the GCN contains $K$ layer and further assume that GCNs on original graph and summary graph share the weight matrix. Consider the last layer,
    \begin{align*}
        \BE^{(K)} &= \sigma(\tilde{\BD}^{-\frac{1}{2}} \tilde{\BA} \tilde{\BD}^{-\frac{1}{2}} \BE^{(K-1)} \BW^{(K-1)}) \\
        \BE_s^{(K)} &= \sigma(\tilde{\BD}_s^{-\frac{1}{2}} \tilde{\BA}_s \tilde{\BD}_s^{-\frac{1}{2}} \BE_s^{(K-1)} \BW^{(K-1)})
    \end{align*}
    The difference of $\BE^{(K)}$ and $\BR \BE_s^{(K)}$ is (refer to Eq.~\eqref{equ:layer_approx}):
    {
        \small
        \begin{align*}
            \left\| \BE^{(K)} - \BR \BE_s^{(K)} \right\| &= \left\| \sigma(\tilde{\BD}^{-\frac{1}{2}} \tilde{\BA} \tilde{\BD}^{-\frac{1}{2}} \BE^{(K-1)} \BW^{(K-1)}) - \sigma(\BR \tilde{\BD}_s^{-\frac{1}{2}} \tilde{\BA}_s \tilde{\BD}_s^{-\frac{1}{2}} \BE_s^{(K-1)} \BW^{(K-1)}) \right\| \\
            \shortintertext{(ReLU function $\sigma$ is a Lipschitz function with Lipschitz constant 1)}
            &\le \left\| \tilde{\BD}^{-\frac{1}{2}} \tilde{\BA} \tilde{\BD}^{-\frac{1}{2}} \BE^{(K-1)} \BW^{(K-1)} - \BR\ \tilde{\BD}_s^{-\frac{1}{2}} \tilde{\BA}_s \tilde{\BD}_s^{-\frac{1}{2}} \BR^\T \BE^{(K-1)} \BW^{(K-1)} \right\| \\
            \shortintertext{(Corollary \ref{coro:symm_laplacian})}
            &= \left\| \tilde{\BD}^{-\frac{1}{2}} \tilde{\BA} \tilde{\BD}^{-\frac{1}{2}} \BE^{(K-1)} \BW^{(K-1)} - \tilde{\BD}^{-\frac{1}{2}} \tilde{\BA}_r \tilde{\BD}^{-\frac{1}{2}} \BE^{(K-1)} \BW^{(K-1)} \right\| \\
            &= \left\| \tilde{\BD}^{-\frac{1}{2}} (\tilde{\BA} - \tilde{\BA}_r) \tilde{\BD}^{-\frac{1}{2}} \BE^{(K-1)} \BW^{(K-1)} \right\| \\
            &\le \left\| \tilde{\BD}^{-\frac{1}{2}} (\tilde{\BA} - \tilde{\BA}_r) \tilde{\BD}^{-\frac{1}{2}}  \right\| \left\| \BE^{(K-1)} \right\| \left\| \BW^{(K-1)} \right\|
        \end{align*}
    }

    Further, for each layer $l$,
    \begin{align*}
        \left\| \BE^{(l)} \right\| &= \left\| \sigma(\tilde{\BD}^{-\frac{1}{2}} \tilde{\BA} \tilde{\BD}^{-\frac{1}{2}} \BE^{(l-1)} \BW^{(l-1)}) \right\| \\
        \intertext{(Note that $\sigma$ is ReLU function and hence $\| \sigma(A) \| \le \| A \|$)}
        &\le \left\| \tilde{\BD}^{-\frac{1}{2}} \tilde{\BA} \tilde{\BD}^{-\frac{1}{2}} \BE^{(l-1)} \BW^{(l-1)} \right\| \\
        &\le \left\| \tilde{\BD}^{-\frac{1}{2}} \tilde{\BA} \tilde{\BD}^{-\frac{1}{2}} \right\|_2 \left\| \BE^{(l-1)} \right\| \left\| \BW^{(l-1)} \right\| \\
        &\le \left\| \BE^{(l-1)} \right\| \left\| \BW^{(l-1)} \right\|
    \end{align*}
    The last inequality comes from the fact that the eigenvalues of a normalized adjacency matrix lies in $[-1, 1]$.

    Combine all layers together, we have:
    {
        \begin{align*}
            \left\| \BE^{(K)} - \BR \BE_s^{(K)} \right\| &\le \left\| \tilde{\BD}^{-\frac{1}{2}} (\tilde{\BA} - \tilde{\BA}_r) \tilde{\BD}^{-\frac{1}{2}} \right\| \left\| \BE^{(K-1)} \right\| \left\| \BW^{(K-1)} \right\| \\
            &\le \left\| \tilde{\BD}^{-\frac{1}{2}} (\tilde{\BA} - \tilde{\BA}_r) \tilde{\BD}^{-\frac{1}{2}} \right\| \left( \prod_{l=0}^{K-1} \left\| \BW^{(l)} \right\| \right) \left\| \BE^{(0)} \right\|
        \end{align*}
    }
\end{proof}

%% file: conclusion.tex
\section{Conclusion}\label{sec:conclusion}
In this paper, we study the problem of learning node embeddings of large graphs via summary graphs theoretically.
We give analysis of three popular embedding methods, DeepWalk, LINE and GCN and reveal that learning embeddings via summary graphs using these three methods is equivalent to learning embeddings on a configuration-based reconstructed graph.
Our analysis can give a theoretical analysis of current existing methods based on heuristic designs. Moreover, our kernel-matrix-based framework is genearl and have great potential be further extended to more graph mining tasks.

Further work includes develop efficient summarization algorithms for the problem based on theory derived in this paper. According to the analysis, one can notice that the approximation error of kernel matrix is closely related to the difference of the normalized adjacency matrix~(Theorem \ref{thm:bound} and \ref{thm:approx_gcn}).
This observation motivates us to \textbf{make the normalized adjacency matrices close}.